\documentclass[final,12pt]{elsarticle}

\usepackage{amssymb}
\usepackage{amsmath,amsfonts}
\usepackage{amsthm}

\journal{Applied and Computational Harmonic Analysis}

\def\sign{{\rm sign}}
\def\t0{{t_0}}

\def\E{{\mathbb E}}

\def\L{{\mathcal L}}

\def\Z{{\mathbb Z}}        
\def\R{{\mathbb R}}        
\def\Prob{{\bf Prob}}
\def\dive{{\rm div}}
\usepackage[colorlinks]{hyperref}
\usepackage{subfigure}

\newtheorem{thm}{Theorem}
\newtheorem{coll}{Corollary}
\newtheorem{lem}{Lemma}

\graphicspath{{./}} 

\begin{document}

\begin{frontmatter}



\title{Analysis of Crowdsourced Sampling Strategies for \\ HodgeRank with Sparse Random Graphs}


\author[utah]{Braxton Osting\corref{CA}}
\address[utah]{Department of Mathematics, University of Utah, Salt Lake City, UT 84112 USA}
\ead{osting@math.utah.edu}
\author[PKU]{Jiechao~Xiong}
\ead{xiongjiechao@pku.edu.cn}
\author[BICMR]{Qianqian~Xu}
\ead{xuqianqian@math.pku.edu.cn}
\address[BICMR]{BICMR, Peking University, Beijing 100871, China}
\author[PKU]{Yuan~Yao\corref{CA}}
\cortext[CA]{Corresponding author.}
\ead{yuany@math.pku.edu.cn}
\address[PKU]{School of Mathematical Sciences, BICMR-LMAM-LMEQF-LMP, Peking University, Beijing 100871, China}


\begin{abstract}
Crowdsourcing platforms are now extensively used for conducting subjective pairwise comparison studies.
 In this setting, a pairwise comparison dataset is typically gathered via random sampling, either \emph{with} or \emph{without} replacement. In this paper, we use tools from random graph theory to analyze these two random sampling methods for the HodgeRank estimator. Using the Fiedler value of the graph as a measurement for estimator stability (informativeness), we provide a new estimate of the Fiedler value for these two random graph models. In the asymptotic limit as the number of vertices tends to infinity, we prove the validity of the estimate. Based on our findings, for a small number of items to be compared, we recommend a two-stage sampling strategy where a greedy sampling method  is used initially and random sampling \emph{without} replacement is used in the second stage. When a large number of items is to be compared, we recommend random sampling with replacement as this is computationally inexpensive and trivially parallelizable. Experiments on synthetic and real-world datasets support our analysis.
\end{abstract}


\end{frontmatter}

\clearpage
\section{Introduction}
\label{Intro}

With the advent of ubiquitous internet access and the growth of crowdsourcing platforms (\emph{e.g.},
\href{https://www.mturk.com}{MTurk},
\href{http://www.innocentive.com/}{InnoCentive},
\href{http://crowdflower.com/}{CrowdFlower},
\href{http://www.crowdrank.net/}{CrowdRank}, and
\href{http://www.allourideas.org/}{AllOurIdeas}),
the crowdsourcing strategy is now employed by a variety of communities.
Crowdsourcing enables researchers to conduct social experiments on a heterogenous set of participants and at a lower economic cost than conventional laboratory studies. For example, researchers can harness internet users to conduct user studies
on their personal computers.
Among various approaches to conduct subjective tests, pairwise comparisons are expected to yield more reliable results. However, in crowdsourced studies, the individuals performing the ratings are diverse compared to more controlled settings, which is difficult to control for using traditional experimental designs; researchers have recently proposed several randomized methods to conduct user studies \cite{added,MM11,tmm12}, which accommodate incomplete and imbalanced data.

HodgeRank, as an application of combinatorial Hodge theory to the preference or rank aggregation problem from pairwise comparison data, possibly being incomplete and imbalanced, was first introduced by \cite{Hodge}, and inspired  a series of studies in statistical ranking \cite{Hirani11,hodge_l1,osting2013enhanced,Osting:2012b}. Hodge theory has also found applications in game theory \cite{Parrilo11_gameflow} and computer vision \cite{Yuan09_hodge,Osher11_retinex}, in addition to traditional applications in fluid mechanics \cite{Chorin93}  \emph{etc}.
HodgeRank formulates the ranking problem in terms of the discrete Hodge decomposition of the pairwise data and shows that it can be decomposed into three orthogonal components:
a gradient flow representing a global rating (optimal in the $L_2$-norm sense),
a triangular curl flow representing local inconsistency, and
a harmonic flow representing global inconsistency.
 Such a perspective generalizes  various linear statistical models to provide a universal geometric description of the structure of paired comparison data, which is possibly incomplete and imbalanced due to crowdsourcing.

The two most popular random sampling schemes in crowdsourcing experiments are random sampling \emph{with} replacement and random sampling \emph{without} replacement. In random sampling \emph{with} replacement, one selects a comparison pair randomly from the whole dataset regardless if the pair has been selected before; whence it is memory free.
In random sampling \emph{without} replacement, each comparison pair in the dataset has an equal chance of being selected; once selected it cannot be chosen again until all possible pairs have been chosen. The simplest model of random sampling \emph{without} replacement in paired comparisons is the Erd\"{o}s-R\'{e}nyi random graph, which is a stochastic process that starts with \emph{n} vertices and no edges, and at each step adds one new edge uniformly \cite{ErdRen59}. As one needs to avoid previous edges, such a sampling scheme is not memory-free and may lead to weak dependence in some estimates.

Recently, \cite{MM11,tmm12} develops the application of HodgeRank with random graph designs in subjective Quality of Experience (QoE) evaluation and shows that random graphs could play an important role in guiding random sampling designs for crowdsourcing experiments. In particular, exploiting topology evolution of clique complexes induced from Erd\"{o}s-R\'{e}nyi random graphs \cite{Kahle09}, \cite{tmm12} shows that at least $O(n\log n)$ distinct random edges are necessary to ensure the inference of a global ranking and $O(n^{3/2})$ distinct random edges are sufficient to remove the global inconsistency.

On the other hand, there are \emph{active sampling schemes} which are designed to maximize the information in the collected dataset, potentially reducing the amount of data collected. Recently, \cite{osting2013enhanced,Osting:2012b} exploits a greedy sampling method to maximize the Fisher information in HodgeRank, which is equivalent to maximizing the smallest nonzero eigenvalue of the unnormalized graph Laplacian (a.k.a. \emph{Fiedler value} or algebraic connectivity ). Although the computational cost of such greedy sampling is prohibitive for large graphs, it effectively boosts the algebraic connectivity compared to Erd\"{o}s-R\'{e}nyi random graphs.

However, active sampling for data acquisition is not always feasible. For example, when data is collected from the Internet crowd or purchasing preferences, data collection is in general passive and independent. An important benefit of random sampling over active methods is that data collection can be trivially parallelized: comparisons can be collected from independent or weakly dependent processes, each selected from a pre-assigned block of object pairs. From this viewpoint, the simplicity of random sampling allows flexibility and applicability  to diverse situations, such as online or distributed ranking,  often desirable for crowdsourcing scenarios.

Therefore, our interest in this paper is to investigate the characteristics of these three sampling methods (\emph{i.e.}, random sampling with/without replacement and greedy sampling) for HodgeRank and identify  an attractive sampling strategy that is particularly suitable for crowdsourcing experiments. The natural questions we are trying to address are: (i) which sampling scheme is the best, \emph{e.g.}, contains the most information for HodgeRank? and
(ii) how do random and greedy sampling schemes compare  in practice?

We approach these problems with a combination of theory and experiment  in this paper. Performance of these sampling schemes is evaluated via the stability of HodgeRank, as measured by the Fiedler value.  The  Erd\"os-R\'enyi random graph model is associated with random sampling without replacement. For this model,  an estimate of the Fiedler value was recently given in \cite{Kolokolnikov2013}. The proof of this estimate  hinges on an estimate of \cite{FeigeOfek05}, which can be shown to  imply that, at first order, the Fiedler value is the minimal degree of the graph. The minimal degree of the graph can then be estimated from the binomial distribution. To analyze the random graph model associated with random sampling with replacement, we generalize the result given in  \cite{FeigeOfek05} to multigraphs. A simple Normal approximation is then used to estimate the Fiedler value. As the graphs become increasingly dense, we prove that  both random sampling methods asymptotically have  the same  Fiedler value. Our  analysis implies:
\begin{enumerate}
\item[i)] For a finite graph which is sparse, random sampling \emph{with} and \emph{without} replacement have similar performance; for a dense finite graph, random sampling \emph{without} replacement is superior to random sampling \emph{with} replacement, and approaches the performance of greedy sampling.
\item[ii)] For very large graphs, the three considered sampling schemes exhibit similar performance.
\end{enumerate}
In particular, the asymptotic behavior of the two random sampling schemes is rigorously proved in Theorem \ref{tm} and their discrepancy for small sample sizes is supported by  heuristic estimates (see Sections \ref{sec:HeurEstMinDegree} and \ref{sec:AsymAnalFiedler}). 
These analytic conclusions and the performance of the greedy sampling strategy are supported by both simulated examples and real-world datasets (see Section \ref{sec:experiments}).  Based on our findings, for a relatively small number of items to be compared, we recommend a two-stage sampling strategy where a greedy sampling method is used initially and random sampling without replacement is used in the second stage.
When a large number of items is to be compared, we recommend random sampling with replacement as this is computationally inexpensive and trivially parallelizable.

\bigskip

\textbf{Outline}. Section \ref{sec:relatedwork}
contains a review of related work. Then we establish some theoretical results for random sampling methods in Section \ref{sec:framework}. Proofs will be collected in  \ref{proof}. The results of detailed experiments on crowdsourced data are reported in Section \ref{sec:experiments}. We conclude in Section \ref{sec:conclusion} with some remarks and a discussion of future work.

\section{Related Work}\label{sec:relatedwork}

\subsection{Crowdsourcing}

The term ``crowdsourcing" is a portmanteau of ``crowd" and ``outsourcing". It is
distinguished from outsourcing in that the work comes from an undefined
public rather than being commissioned from a specific, named group. The benefits of crowdsourcing include  time-efficiency and low monetary costs. Among various crowdsourcing platforms, Amazon's Mechanical Turk (\href{https://www.mturk.com}{MTurk}) is probably the most popular and provides a marketplace for a variety of tasks; anyone seeking help from the Internet crowd can post their task requests to the website. 
Another platform, \href{http://www.innocentive.com/}{Innocentive},  enables organizations to engage diverse innovation communities such as employees, partners, or customers to rapidly generate novel ideas and innovative solutions to challenging research and development
problems. \href{http://crowdflower.com/}{CrowdFlower}'s expertise is in harnessing the Internet crowd to provide a wide range of enterprise solutions, taking complicated projects and dividing  them  into smaller, simpler tasks, which are then completed by individual contributors.
\href{http://www.crowdrank.net/}{CrowdRank} is an innovative platform that draws on the over 3 million community votes already cast to bring the crowdsourcing revolution to rankings via a novel pairwise ranking methodology that avoids the tedium of asking community members to rank every item in a category. In addition, \href{http://www.allourideas.org/}{Allourideas} provides  a free and open-source website that allows groups all over the world to create and use pairwise wiki surveys. Respondents can either participate in a pairwise wiki survey or add new
items that are then presented to future respondents.

With the help of these platforms, requesters post tasks ({\it e.g.} image annotation \cite{mm34,MIR10},
document relevance \cite{mm2},
document evaluation \cite{mm25},
music emotion recognition \cite{MM13workshop1},
affection mining in computer games \cite{MM13workshop2}, 
and quality of experience evaluation \cite{MM09,tmm12}) 
and users are compensated in the form of micro-payments
for completing these posted tasks. 
Several studies have been conducted to evaluate  the quality of completed tasks obtained from crowdsourcing approaches. For example,  researchers have investigated the reliability of non-experts  and found that a single expert in the majority of cases is more reliable than a non-expert. However, using an aggregate of several, cheap non-expert judgements could approximate the performance of expensive expertise \cite{non-expert1,non-expert2}. From this point of view, conducting subjective tests in a crowdsourcing context is a reasonable strategy.

\subsection{Pairwise ranking aggregation}
The problem of ranking or rating with paired comparison data has been widely studied in a variety of fields including decision science \cite{yao35},
machine learning \cite{yao20}, social choice \cite{yao2}, and
statistics \cite{yao26}. Various methods have been studied for this problem, which, among others,  includes 
maximumlikelihood under a Bradley-Terry model, 
rank centrality (PageRank/MC3) \cite{negahban2012iterative,dwork2001rank}, 
HodgeRank \cite{Hodge},
and a pairwise variant of the Borda count \cite{de1781memoire, Hodge}. 
If we consider the setting where pairwise comparisons are drawn
I.I.D. from some fixed but unknown probability distribution, under a ``time-reversibility" condition, the rank centrality (PageRank) and
HodgeRank algorithms both converge to an optimal ranking \cite{ICML14}.  However, PageRank is only able to aggregate the pairwise comparisons
into a global ranking over the items. HodgeRank not only provides a means to determine a global ranking from paired comparison data under
various statistical models (e.g., Uniform, Thurstone-Mosteller, Bradley-Terry, and Angular Transform),
but also measures the inconsistency of the global ranking obtained. In particular, it takes a graph theoretic view, which maps paired comparison data to edge
flows on a graph, possibly imbalanced (where different pairs
may receive different number of comparisons) and incomplete
(where every participant may only provide partial comparisons),
and then applies combinatorial Hodge Theory to
achieve an orthogonal decomposition of such edge flows into
three components: a gradient flow representing the global rating (optimal
in the $L_2$-norm sense), a triangular curl flow representing  local inconsistency,
and a harmonic flow representing global inconsistency. In this paper, we will analyze two random sampling methods based on the HodgeRank estimate.

\subsection{Active sampling}

The fundamental notion of active sampling has a long history in machine learning. To our knowledge, the first to discuss it explicitly were \cite{1974} and \cite{1975}.
Subsequently, the term active learning was coined \cite{1994} and has
been shown to benefit a number of multimedia applications
such as object categorization \cite{cvpr20}, image retrieval \cite{cvpr16,cvpr37}, video classification \cite{cvpr36}, dataset annotation \cite{cvpr8}, and interactive co-segementation \cite{cvpr4}, maximizing
the knowledge gain while valuing the user effort \cite{cvpr35}.

Recently, several authors have studied the active sampling problems for ranking and rating, with the goal of reducing the amount of data that must be collected.
For example, \cite{jamieson2011active} considers the case when the true scoring function reflects the Euclidean distance of object covariates from a global reference point. If objects are embedded in $\mathbb R^d$ or the scoring function is linear in such a space, the active sampling complexity can be reduced to $O(d \log n)$, as demonstrated through a comparison of beer  \cite{jamiesonactive}.
Moreover, \cite{ailon2012active} discusses the application of a polynomial time approximate solution (PTAS) for the NP-hard minimum feedback arc-set (MFAST) problem, in active ranking with sample complexity $O(n \cdot {\rm{poly}}(\log n, 1/\varepsilon))$ to achieve $\varepsilon$-optimum. 
The works mentioned above can be treated as ``learning to rank" which requires a vector representation of the items to be ranked, thus can not be directly applied to crowdsourced ranking. In the crowdsourcing scenario, the explicit feature representation
of items is unavailable and the goal becomes to learn a single ranking
function from the ranked items using a smaller number of samples selected actively \cite{negahban2012iterative}. In \cite{chen2013pairwise}, a Bayesian framework is proposed to actively select pairwise comparison queries based on Bradley-Terry models. Furthermore, \cite{Chen2015SDM} addresses the problem of budget allocation in crowd labeling using the Bayesian Markov decision process and characterizing the optimal policy using the dynamic programming.
Most recently, \cite{osting2013enhanced,Osting:2012b} approaches active sampling from a
statistical perspective of maximizing the Fisher information, which they show to be equivalent to maximizing the Fiedler value of the graph (smallest nonzero eigenvalue of the graph Laplacian which arises in HodgeRank), subject to an integer weight constraint. In this paper, we shall focus on analyzing the Fiedler value of graphs generated based on random sampling schemes.

\section{Analysis of sampling methods}\label{sec:framework}
Statistical preference aggregation or ranking/rating from pairwise comparison data is a classical problem, which can be traced back to the $18^{\text th}$ century with  discussions on voting and social choice. This subject area has recently undergone  rapid growth in various applications due to the availability of the Internet and development of crowdsourcing techniques. In these scenarios, typically we are given pairwise comparison data on a graph $G = (V,E)$, $Y^\alpha \colon E\rightarrow \mathbb{R}$ such that $Y_{ij}^\alpha=-Y_{ji}^\alpha$ where $\alpha$ is the index for multiple comparisons, $Y_{ij}^\alpha>0$ if it prefers $i$ to $j$ and $Y_{ij}^{\alpha}\leq 0$ otherwise. In the dichotomous choice, $Y_{ij}^\alpha$ can be taken as $\{\pm 1\}$, while multiple choices are also widely used  (\emph{e.g.}, $k$-point Likert scale, $k=3,4,5$).

The general purpose of preference aggregation is to look for a global score $x\colon V\to \R$ such that
\begin{equation} \label{eq:ho_rank0}
\min_{\substack{x\in {\mathbb{R}}^{|V|} \\ x\bot1 }} \sum_{i,j,\alpha} \omega_{ij}^\alpha \L(x_i - x_j, Y_{ij}^\alpha),
\end{equation}
where $\L(x,y)\colon \R\times \R\to \R$ is a loss function, $\omega_{ij}^\alpha$ denotes the confidence weight of this comparison which is set to be 1 in this paper but other choices are also possible, and $x_i$ ($x_j$) represents the global ranking score of item \emph{i} (\emph{j}, respectively).  For a connected graph $G$, restricting to the subspace $\{x \in \R^{|V|} \colon x\bot 1\}$ guarantees a unique solution in \eqref{eq:ho_rank0}. For example, $\L(x,y)= (\sign(x) - y)^2$ leads to the minimum feedback arc-set (MFAST) problem which is NP-hard, where \cite{ailon2012active} proposes an active sampling scheme whose complexity is $O(n \cdot {\rm{poly}}(\log n, 1/\varepsilon))$ to achieve $\varepsilon$-optimum. In HodgeRank, one benefits from the use of square loss $\L(x,y)=(x-y)^2$ which leads to fast algorithms to find optimal global ranking $x$, as well as an orthogonal decomposition of the least square residue into local and global inconsistencies \cite{Hodge}.

To see this, let $\hat{Y}_{ij}=( \sum_\alpha \omega_{ij}^\alpha Y^\alpha_{ij}) / (\sum_\alpha \omega_{ij}^\alpha )$ ($\omega_{ij} = \sum_\alpha \omega^\alpha_{ij}$) be the mean pairwise comparison scores on $(i,j)$, which can be extended to a family of generalized statistical linear models. To characterize the solution and residual  of (\ref{eq:ho_rank0}), we first define a 3-clique complex $X_G=(V,E,T)$ where $T$ collects all triangular complete subgraphs in $G$:
$$
T=\left\{\{i,j,k\} \in
\begin{pmatrix} V\\3 \end {pmatrix}
\colon \{i,j\},\{j,k\},\{k,i\}\in  E\right\}.
$$
Then every $\hat{Y}$ admits an orthogonal decomposition:
\begin{equation} \label{hodge}
\hat{Y} = \hat{Y}^{g} + \hat{Y}^{h} + \hat{Y}^{c},
\end{equation}
where the gradient flow $\hat{Y}^{g}$ satisfies
\begin{equation} \label{eq:gradient-flow}
\hat{Y}^{g}_{ij}  = x_i - x_j,\ \ \textrm{for some}\ x\in {\mathbb R}^n ,
\end{equation}
the harmonic flow $\hat{Y}^{h}$ satisfies
\begin{equation} \label{eq:curl-free}
\hat{Y}^{h}_{ij} + \hat{Y}^{h}_{jk} + \hat{Y}^{h}_{ki} = 0, \ \textrm{for each $\{i,j,k\}\in T$} ,
\end{equation}
\begin{equation} \label{eq:divergence-free}
\sum_{{j: (i,j)\in E}} \omega_{ij}\hat{Y}^{h}_{ij} = 0, \ \textrm{for each $i\in V$},
\end{equation}
and the curl flow $\hat{Y}^{c}$  satisfies (\ref{eq:divergence-free}) but not (\ref{eq:curl-free}).

The residuals $\hat{Y}^{c}$ and $\hat{Y}^{h}$ indicate  whether inconsistencies in the ranking data arises locally or globally. Local inconsistency can be fully characterized by triangular cycles (\emph{e.g}. $i \succ j \succ k\succ i $), while global inconsistency generically involves longer cycles of inconsistency in $V$ (\emph{e.g}. $i \succ j \succ k \succ \cdots \succ i$), which may arise due to data incompleteness and cause the fixed tournament issue. Random sampling to avoid global inconsistency, generally requires  at least $O(n^{3/2})$ random samples without replacement \cite{MM11,tmm12}.

The global rating score $x$ in (\ref{eq:gradient-flow}) can be obtained  by solving the normal equation \cite{Hodge},
\begin{equation}
\label{eq:HodgeRank}
Lx = b.
\end{equation}
Here, $L=D-A$ is the unnormalized graph Laplacian, where $A(i,j)=\omega_{ij}$ if $(i,j)\in E$, $A(i,j)=0$ otherwise, and $D$ is a diagonal matrix with $D(i,i)=\sum_{{j: (i,j)\in E}} \omega_{ij}$, as well as $b= \dive (\hat{Y})$ is the divergence flow defined by $b_i = \sum_{{{j: (i,j)\in E}}}\omega_{ij} \hat{Y}_{ij}$. There is an extensive literature in linear algebra on solving the symmetric Laplacian equation. However, all methods are subject to the intrinsic stability of HodgeRank,  characterized in the following subsection.

\subsection{Stability of HodgeRank}
The following classical result  (see, {\it e.g.}, \cite{MatrixCompIII}) gives a measure of the sensitivity of the global ranking score $x$ against perturbations on $L$ and $b$. Given the parameterized system
\[(L+\epsilon F)x(\epsilon) = b + \epsilon f, ~~~~ x(0) = x\]
where $F\in R^{n\times n}$ and $f \in R^n$,  then
\[\frac{\|x(\epsilon) - x\|}{\|x\|} \le |\epsilon|\| \ L^{-1}\| \left(\frac{\|f\|}{\|x\|}+\|F\| \right)+O(\epsilon^2).\]
 Here and throughout this paper, the matrix norm is the spectral norm and the vector norm is the Euclidean norm. In crowdsourcing, the matrix $L$ is determined by the sampled pairs, and can be regard as fixed when given the pairwise data. However,  $b=\dive(\hat{Y})$ is random because of  noise possibly  induced by crowdsourcing. So there is no perturbation on $L$, \emph{i.e.}, $F=0$, and we obtain
\begin{equation}
\label{eq:sensitivity}
\frac{ \|x(\epsilon) - x\|}{\|f\|} \le  \|L^{-1}\|  \ |\epsilon|+O(\epsilon^2).
\end{equation}
If the graph representing the pairwise comparison data is connected, the graph Laplacian, $L$, has a one-dimensional kernel spanned by the constant vector. In this case, the solution to \eqref{eq:HodgeRank} is understood in the minimal norm least-squares sense, \emph{i.e.} $\hat{x} = L^\dag b$ where $L^\dag$ is the Moore-Penrose pseudoinverse of $L$. Hence \eqref{eq:sensitivity} implies that the sensitivity of the estimator is controlled by $\| L^\dag \| = [ \lambda_2(L) ]^{-1}$, the reciprocal of the second smallest eigenvalue of the graph Laplacian. $\lambda_2(L)$ is also referred to as the \emph{Fiedler value} or \emph{algebraic connectivity} of the graph. It follows that collecting pairwise comparison data so that $\lambda_2(L)$ is large provides an estimator which is insensitive to noise in the pairwise comparison data, $\hat{Y}$.

\textbf{Remark.} \cite{osting2013enhanced,Osting:2012b} show that for a fixed variance statistical error model, the Fisher information matrix of the HodgeRank estimator \eqref{eq:HodgeRank} is proportional to the graph Laplacian, $L$. Thus, finding a graph with large algebraic connectivity, $\lambda_2(L)$, can be  equivalently viewed in the context of optimal experimental design as maximizing the  ``E-criterion'' of the Fisher information.

\subsection{Random sampling schemes}
 In what follows, we study two random sampling schemes:
\begin{enumerate}
\item $G_0(n,m)$: \emph{Uniform sampling with replacement}. Each edge is sampled from the uniform distribution on $\binom{n}{2}$ edges, with replacement. This is a weighted graph and the sum of weights is $m$.
\item $G(n,m)$: \emph{Uniform sampling without replacement}. Each edge is sampled from the uniform distribution on the available edges without replacement. For $m\leq \binom{n}{2}$, this is an instance of the Erd\"{o}s-R\'{e}nyi random graph model $G(n,p)$ with $p=m/\binom{n}{2}$.
\end{enumerate}
Motivated by the estimate given in \eqref{eq:sensitivity}, we will characterize the behavior of the Fiedler value of the graph Laplacians associated with these sampling methods. It is well-known that the Erdos-Renyi random graph $G(n,p)$ is connected with high probability if the sampling rate is at least $p=(1+\epsilon) \log n/n$ \cite{ErdRen59}. Therefore we use  the parameter $p_0:=2m/ ((n-1) \log n)\geq 1$ (where {$m=n(n-1) p/2\approx  n^2 p/2$}), the degree above the connectivity threshold, to compare the efficiency in boosting Fiedler values for different sampling methods.

As a comparison for random sampling schemes, we consider a \emph{greedy sampling} method of sampling pairwise comparisons to maximize the algebraic connectivity of the graph \cite{Ghosh:2006b,osting2013enhanced,Osting:2012b}.
The problem of finding a set of $m$ edges on $n$ vertices with maximal algebraic connectivity is an NP-hard problem. The following greedy heuristic, based on the Fiedler vector, $\psi$, can be used. The Fiedler vector is the eigenfunction of the graph Laplacian corresponding to the Fiedler value. We shall denote the graph with $m$ edges on  $n$ vertices by $G_\star(n,m)$.  The graph is built iteratively, at each iteration, the Fiedler vector is computed and  the edge $(i,j)$ which maximizes $(\psi_{i} - \psi_{j})^{2}$ is added to the graph. The iterates are repeated until a graph of the desired sized is obtained.

\subsection{Fiedler value and minimal degree}
The key to evaluating the Fiedler value of random graphs is via the graph minimal degree, $d_{\text{min}}$. This is due to the definition of graph Laplacian,
$$L=D-A, $$
whose diagonal $D(i,i) \simeq O(m/n)$ dominates as $\max_{\|v\|=1} v^T A v \simeq O(\sqrt{m/n})$. The following Lemma makes this observation precise, which is used by \cite{Kolokolnikov2013} in the study of Erd\"{o}s-R\'{e}nyi random graphs.
\begin{lem}\label{lm1} Consider the random graph $G_0(n,m)$ (or $G(n,m)$) and let $\lambda_2$ be the Fiedler value of the graph. Suppose there exists a $p_0 > 1$ so that $2m\ge p_0n\log n$ and $C, c_1> 0$ so that
\[|d_{\text{min}}-c_1\frac{2m}{n}|\le C\sqrt{\frac{2m}{n}}\]
with probability at least $1-O(e^{-\Omega(\sqrt{\frac{2m}{n}})})$. Then there exists a $\tilde{C} > 0$ so that
\[|\lambda_2-c_1\frac{2m}{n}|\le \tilde{C}\sqrt{\frac{2m}{n}}. \]
\end{lem}
Lemma \ref{lm1} implies that the difference between $\lambda_2$ and $d_{\text{min}}$ (i.e., minimal degree) is  small, so the  Fiedler value for both random graphs can be approximated by their minimal degrees.

The proof for $G(n,m)$ follows from \cite{Kolokolnikov2013},  which establishes the result for the Erd\"{o}s-R\'{e}nyi random model $G(n,p)$. The proof  for  $G_0(n,m)$, needs the following lemma.

\begin{lem}\label{lm2}
Let $A$ denote the adjacency matrix of a random graph from $G_0(n,m)$ and $S = \{v\bot1\colon \|v\|=1\}$. There exists a constant $c>0$, such that if $m>n\log n/2$, the estimate
\[ \max_{v\in S}  v^TAv \ \le \  c\sqrt{2m/n} \]
holds with probability at least $1-O(1/n).$
\end{lem}

With the aid of this lemma, one can estimate the Fiedler value by the minimal degree of $G_0(n,m)$. In fact,
\begin{eqnarray*}
\lambda_2(L) &=& \min_{v\in S} \langle v,Lv\rangle\\ \nonumber
 &=&\min_{v\in S} \langle v,Dv\rangle - \langle v,Av\rangle\\\nonumber
 &\ge& d_{\text{min}} - \max_{v\in S} \langle v,Av\rangle.   \nonumber
 \end{eqnarray*}
 Also Cheeger's inequality tells that $\lambda_2(L) \le \frac{n}{n-1}d_{\text{min}}$. These bounds show the validity of Lemma \ref{lm1}. The proof of Lemma \ref{lm2} is given in \ref{proof}.


\subsection{A heuristic estimate of the minimal degree} \label{sec:HeurEstMinDegree}
In this section, we estimate  the minimal degree. First, consider the Erd\"{o}s-R\'{e}nyi random graph model $G(n,p)$ with $p=m/\binom{n}{2}$. Then $d_i \sim B(n,p)$, so $\frac{d_i-np}{\sqrt{np(1-p)}}\sim N(0,1)$. The degrees are \emph{weakly dependent}. If the degrees were \emph{independent}, the following concentration inequality for Gaussian random variables,
\[ \Prob(\max_{1\leq i \leq n} |X_i|>t) \leq n \exp\left(-\frac{t^2}{2}\right), \quad  X\sim N(0,I_n), \]
 would imply that the minimal value of $n$ copies of $N(0,1)$ is about $-\sqrt{2\log n}$. In this case,
\[d_{\text{min}} \approx np - \sqrt{2\log (n) np(1-p)},\]
implying that
\[\frac{d_{\text{min}}}{np} \approx 1 - \sqrt{\frac{2\log n}{np}}\sqrt{1-p}.\]

A similar approximation can be employed  for $G_0(n,m)$. Here,
$d_i \sim B(m,2/n)$, so $\frac{d_i-np}{\sqrt{np(1-2/n)}}\sim N(0,1)$. Again, the $d_i$ are only weakly  dependent, so
\[d_{\text{min}} \approx np - \sqrt{2\log ( n ) np(1-2/n)},\]
 which implies that
\[ \frac{d_{\text{min}}}{np} \approx 1 - \sqrt{\frac{2\log n}{np}}\sqrt{1-2/n}. \]
Collecting these results and using $\frac{d_{min}}{np}\simeq \frac{\lambda_2}{np}$, we have the following estimates.

{\bf Key Estimates}.
\begin{align}
\label{eq:G0lam}
G_0(n,m)\colon \ & \frac{\lambda_2}{np} \approx a_1(p_0,n) := 1 - \sqrt{\frac{2}{p_0}}\sqrt{1-\frac{2}{n}} \\
\label{eq:Glam}
G(n,m)\colon \ &  \frac{\lambda_2}{np} \approx a_2(p_0,n):= 1 - \sqrt{\frac{2}{p_0}}\sqrt{1-p}
\end{align}
where $p=\frac{p_0 \log n}{n}$.

\textbf{Remark.} As $n\to\infty$, both \eqref{eq:G0lam} and \eqref{eq:Glam} become  $\frac{\lambda_2}{np} \approx  1 - \sqrt{\frac{2}{p_0}}$. But for finite $n$ and dense $p$, $G(n,m)$ may have larger Fiedler value than $G_0(n,m)$.

The above reasoning  (falsely) assumes independence of $d_i$, which is only valid as $n\to \infty$. In the following section, we make this precise with an asymptotic estimate of the Fiedler value in the two random sampling schemes.

\subsection{Asymptotic analysis of the Fiedler value} \label{sec:AsymAnalFiedler}
In the last section, we gave a heuristic estimator of the Fiedler value. The following theorem gives an asymptotic estimate of the Fiedler value as $n\to \infty$.

\begin{thm}\label{tm}
Consider a random graph $G_0(n,m)$ (or $G(n,m)$) on n vertices corresponding to uniform sampling with (without) replacement and $m = p_0n\log (n)/2 $. Let $\lambda_2$ be the Fiedler value of the graph. Then
\begin{equation}\label{eq:lam2bign}
\frac{\lambda_2}{2m/n} \simeq a(p_0) + O(\frac{1}{\sqrt{2m/n}}),
\end{equation}
with high probability, where $a(p_0)\in(0,1)$ denotes the solution to
\[p_0-1 = ap_0(1-\log a). \]
\end{thm}

The proof of Theorem \ref{tm} is given in \ref{proof}.

\textbf{Remark.} For $p_0 \gg 1$, $a(p_0)=1-\sqrt{2/p_0} +O(1/p_0)$ \cite{Kolokolnikov2013}. Thus, Theorem  \ref{tm} implies  that  the two sampling methods have the same asymptotic algebraic connectivity, $\frac{\lambda_2}{2m/n} \simeq 1-\sqrt{2/p_0}$ as $n\to \infty$ and $p_0\gg 1$. Note from \eqref{eq:G0lam} and \eqref{eq:Glam}, that $\lim_{n\to \infty} a_1(p_0,n) = \lim_{n\to \infty} a_2(p_0,n) = 1 - \sqrt{2/p_0}$.

\begin{figure}[t!]
\begin{center}
\includegraphics[width=0.7\columnwidth]{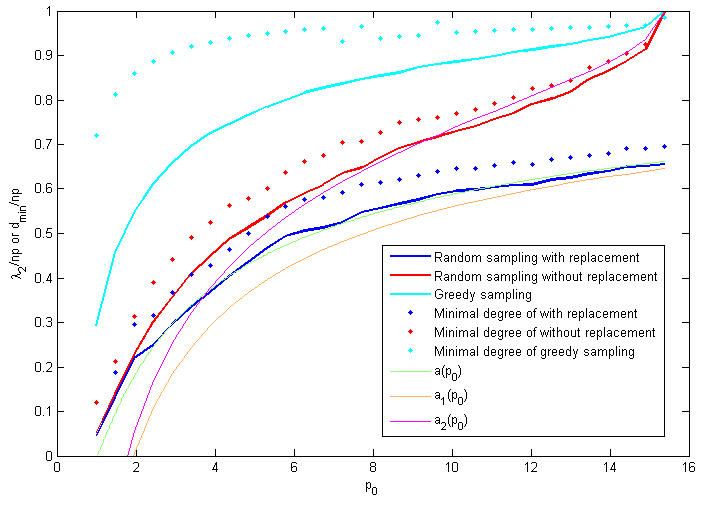}
\caption{A comparison of the Fiedler value, minimal degree, and estimates  \eqref{eq:G0lam}, \eqref{eq:Glam}, and \eqref{eq:lam2bign} for graphs generated via random sampling with/without replacement and greedy sampling for $n= 64$.}
\label{icml-simulatedvalid}
\end{center}
\vskip -0.2in
\end{figure}

\begin{figure}
 \begin{center}
 \subfigure [$n = 2^5$]{
\includegraphics[width=0.48\columnwidth]{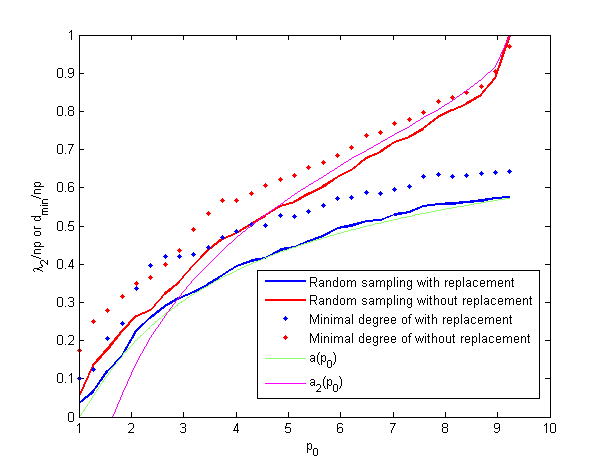}}
 \subfigure[$n = 2^7$]{
\includegraphics[width=0.48\columnwidth]{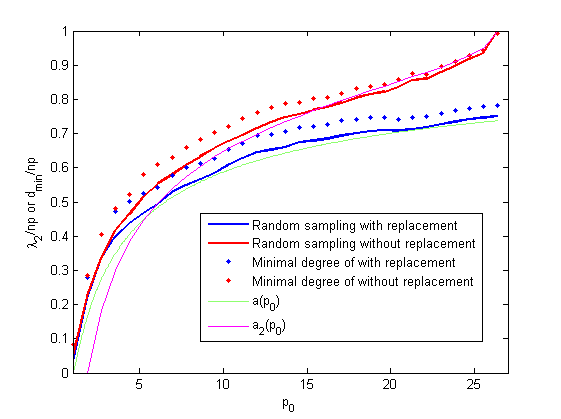}}
 \subfigure[$n = 2^9$]{
\includegraphics[width=0.48\columnwidth]{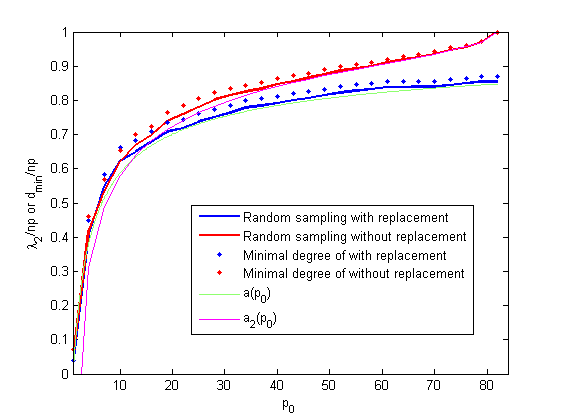}}
 \subfigure[$n = 2^{11}$]{
\includegraphics[width=0.48\columnwidth]{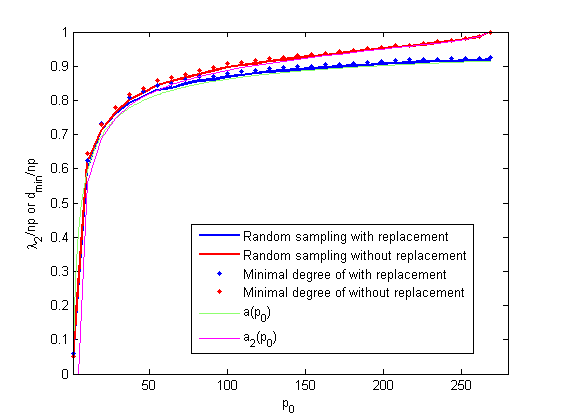}}
   \caption{Algebraic connectivity and minimal degree: Random sampling with replacement vs. Random sampling without replacement for $n = 2^5$, $2^7$, $2^9$, and $2^{11}$. The gaps among these sampling schemes vanish as $n\to \infty$.} \label{icml-simulatedvalid2}
\end{center}
\end{figure}


The main difference between $G(n,p)$ and $G_0(n,m)$ is the weak dependence pattern; the dependence of $d_i$ and $d_j$ only occurs on edge $(i,j)$ which only appears at most once for $G(n,p)$, but all of the $m$ edges can be $(i,j)$ for $G_0(n,m)$. However, we still have $d_i$ and $d_j$ are almost independent when $n$ is sufficiently  large, so the heuristic estimator using I.I.D. Normal distribution as an approximation is not unreasonable.

Theorem \ref{tm} is supported by Figures \ref{icml-simulatedvalid} and \ref{icml-simulatedvalid2},  where the Fiedler value, minimal degree, and various estimates,  \eqref{eq:G0lam}, \eqref{eq:Glam}, and \eqref{eq:lam2bign}, are plotted for varying edge sparsity, $p_0$.
For $G_0(n,m)$, we observe that $a(p_0)$ fits $d_{\min}$ and $\lambda_2$ pretty well for all $p_0$.
For $G(n,m)$, we observe that $a(p_0)$ fits $d_{\text{min}}$ and $\lambda_2$ well when $p_0$ is small, but when $p_0$ is large, the estimate give in \eqref{eq:Glam} is more reliable. In all cases, the Fiedler value for the graph, $G_\star$, generated by greedy sampling, is larger than that for randomly sampled graphs.

\section{Experiments} \label{sec:experiments}
In this section, we study three examples with both simulated and real-world data to illustrate the validity of the analysis above and applications of the proposed sampling schemes. The first example is with simulated data while the latter two consider real-world data from QoE evaluation. The code for the numerical experiment and the real-world datasets can be downloaded from \href{https://code.google.com/p/active-random-joint-sampling/}{https://code.google.com/p/active-random-joint-sampling/}.

\subsection{Simulated data}
This subsection uses simulated data to illustrate  the performance differences among the three sampling schemes. We randomly create a global ranking score as the ground-truth, uniformly distributed on [0,1] for $|V| =n$ candidates. In this way, we obtain a complete graph with ${ n \choose 2 }$ edges consistent with the true preference direction. We sample pairs from this complete graph using random sampling with/without replacement and the greedy sampling scheme. The experiments are repeated 1000 times and  ensemble statistics for the HodgeRank estimator \eqref{eq:HodgeRank} are recorded.  As we know the ground-truth score, the metric used here is the $L_2$-distance between the HodgeRank estimate and ground-true score, $\|\hat{x}-x^*\|$.

Figure \ref{icml-simulated1} (a) shows the mean $L_2$-distance and standard deviation  associated with the three sampling schemes for $n = 16$ (chosen to be consistent with the two real-world datasets considered later). The $x$-axes of the graphs are the number of edges, as measured by $p_0 = \frac{pn}{\log n} $, taken to be greater than one so that the graph is connected with high probability. From these experimental results, we observe  that the performance of random sampling without replacement is better than random sampling with replacement in all cases with smaller $L_2$-distance and smaller standard deviation.
As $p_0$ grows, the performance of the two random sampling schemes diverge. When the graph is sparse, the greedy sampling scheme shows better performance than random sampling with/without replacement. However, when the graphs become dense, random sampling without replacement performs qualitatively similar to greedy sampling.

\begin{figure}[t]
 \begin{center}
 \subfigure [$OP = 0\%$]{
\includegraphics[width=0.3\textwidth,height=9cm]{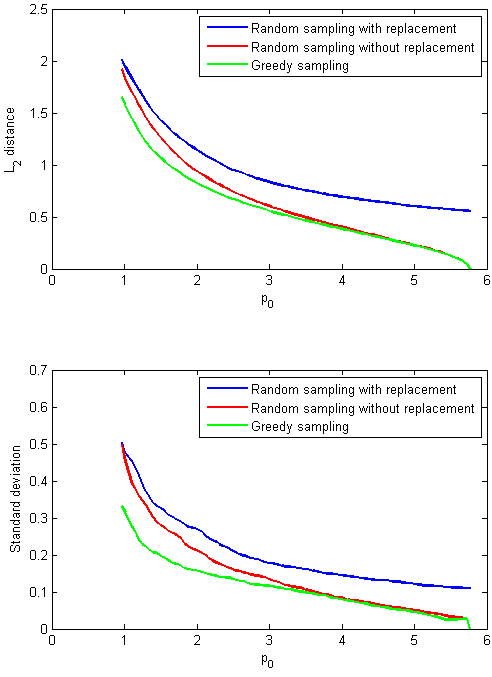}}
 \subfigure[$OP = 10\%$]{
\includegraphics[width=0.3\textwidth,height=9cm]{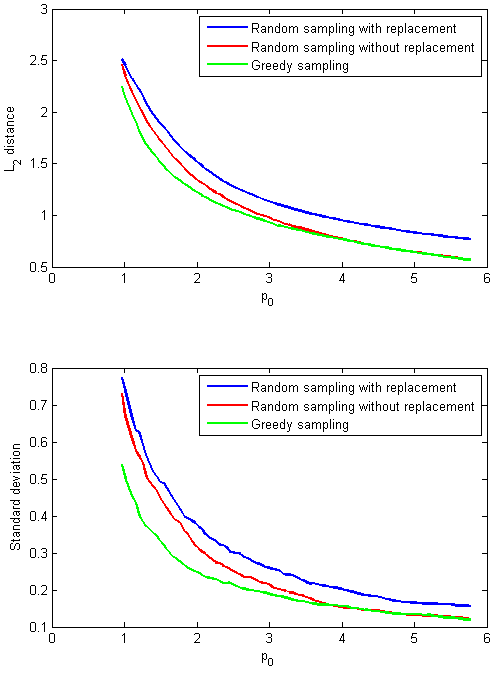}}
 \subfigure[$OP = 30\%$]{
\includegraphics[width=0.3\textwidth,height=9cm]{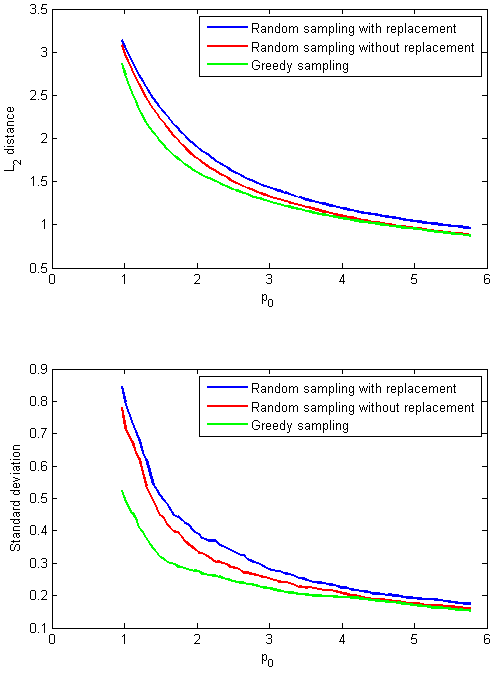}}
   \caption{The $L_2$-distance and standard deviation between ground-truth and HodgeRank estimate for
random sampling with/without replacement and greedy sampling for $n = 16$.  
} \label{icml-simulated1}
\end{center}
\end{figure}



To simulate real-world data contaminated by outliers, each binary comparison is independently flipped with a probability, referred to as \emph{outlier percentage} (OP). For $n=16$, with $\text{OP}=10\%$ and $30\%$,  we plot the number of sampled pairs against the $L_2$-distance and standard deviation between the ground-truth and HodgeRank estimate in Figure \ref{icml-simulated1} (b,c). As in the non-contaminated case, the greedy sampling strategy outperforms the random sampling strategy.  As OP increases, the performance gap among the three sampling schemes decreases.

\begin{figure*}[t!]
\begin{center}
\includegraphics[width=\textwidth]{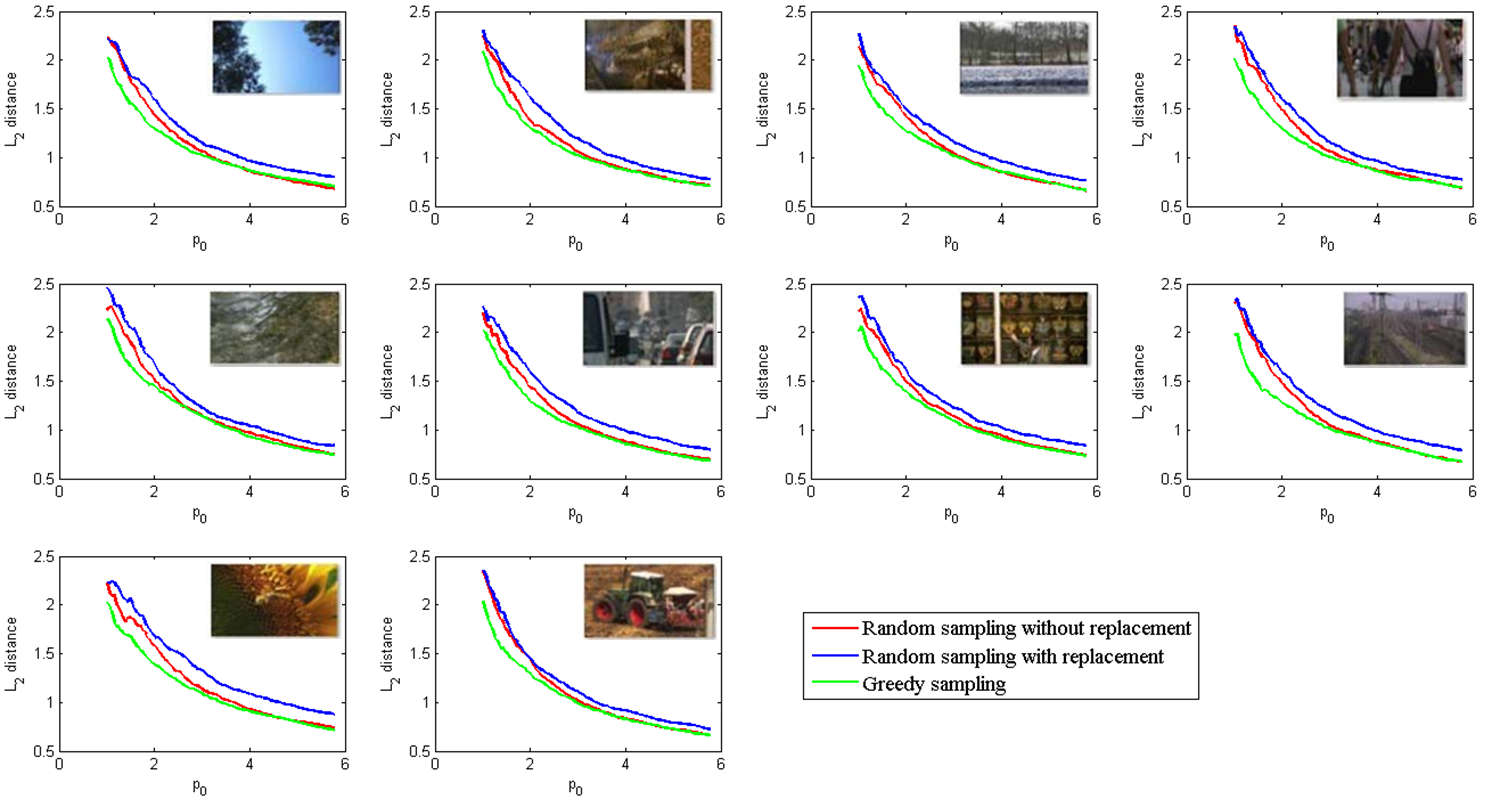}

\caption{Random sampling with/without replacement vs. greedy sampling for 10 reference videos in LIVE database \cite{LIVE}.}
\label{qoe}
\end{center}

\end{figure*}

\begin{figure}[t]
\begin{center}
\includegraphics[width= \textwidth]{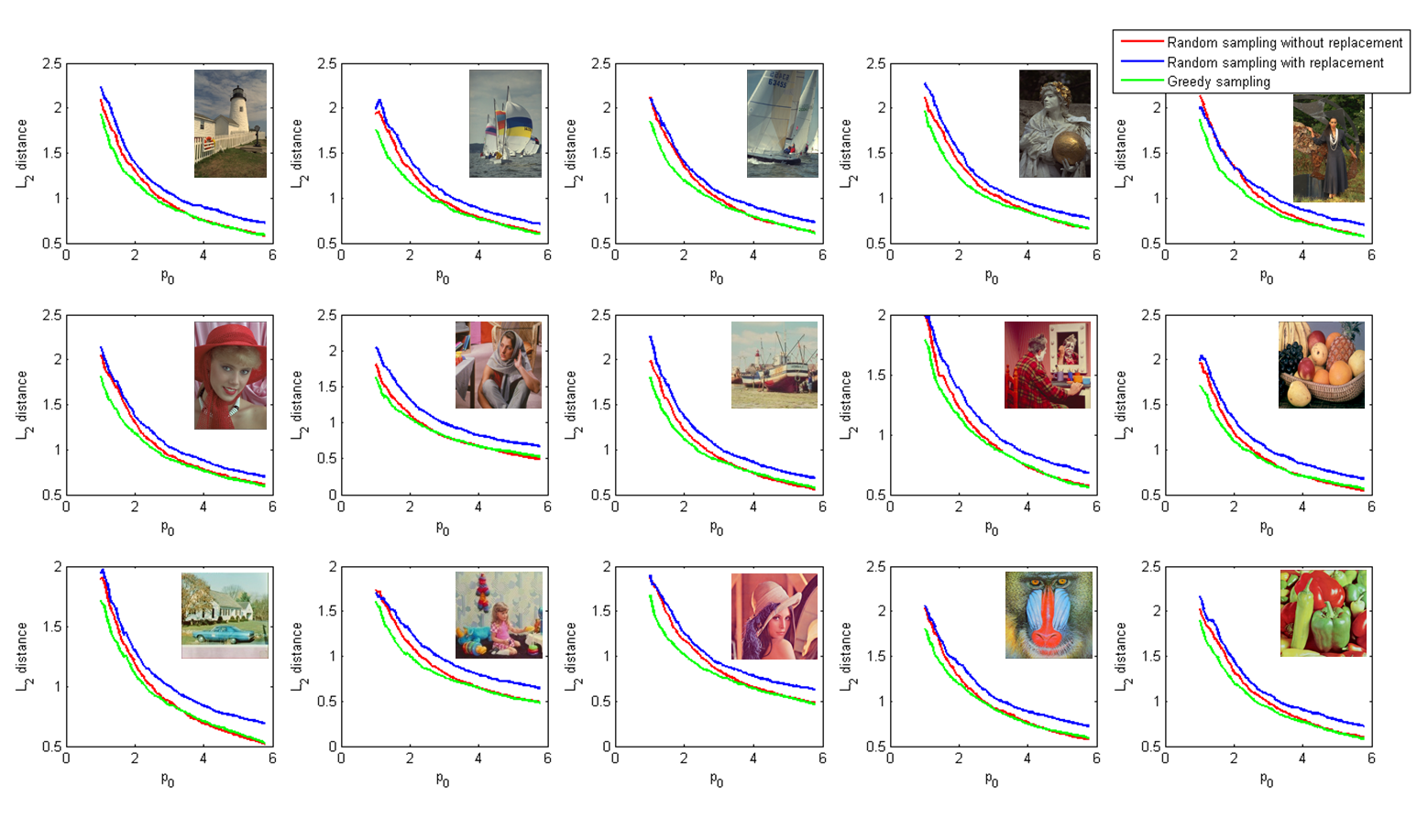}
\caption{Random sampling with/without replacement vs. greedy sampling for 15 reference images in LIVE \cite{LIVE} and IVC \cite{IVC} databases.}
\label{online}
\end{center}
\end{figure}

\subsection{Real-world data}

The second example gives a comparison
of the three sampling methods on a video quality assessment dataset \cite{MM11}.
It contains 38,400 paired comparisons of the LIVE dataset \cite{LIVE} from 209 random observers. An attractive property of this dataset is that the paired comparison data is complete and balanced. As LIVE includes 10 different reference videos and 15 distorted versions of each reference (obtained using four different
distortion processes --- MPEG-2 compression, H.264 compression, lossy
transmission of H.264 compressed bitstreams through simulated IP
networks, and lossy transmission of H.264 compressed bitstreams
through simulated wireless networks),
for a total of 160 videos, the complete comparisons of this video database
requires  $10 \times { 16 \choose 2 }  =1200 $ comparisons. Therefore, 38,400 comparisons correspond to 32 complete rounds.

As there is no ground-truth scores available, results obtained from all the paired comparisons are treated as the ground-truth. To ensure the statistical stability,
for each of the 10 reference videos, we sample using each of the three methods 100 times. Figure \ref{qoe} shows the experimental results of the 10 reference videos in LIVE database \cite{LIVE}). It is interesting to obtain similar observations on all of
these large scale data collections. Consistent with the simulated data, when the graph is sparse, greedy sampling performs better than both random sampling schemes; as the number of samples increases, random sampling without replacement exhibits similar performance in the prediction of global ranking scores.

The third example shows the sampling results on an imbalanced dataset for image quality assessment, which contains 15 reference images and 15 distorted versions of each reference, for a total of 240 images which come from two publicly available datasets, LIVE
\cite{LIVE} and IVC \cite{IVC}. The distorted images in LIVE dataset \cite{LIVE}
are obtained using five different distortion processes ---
JPEG2000, JPEG, White Noise, Gaussian Blur, and Fast
Fading Rayleigh, while the distorted images in IVC dataset \cite{IVC} are derived from
four distortion types --- JPEG2000, JPEG, LAR Coding, and
Blurring. In total, 328 observers,
each of whom performs a varied number of comparisons via the Internet, provide
43,266 paired comparisons. Since the number of paired comparisons in the is dataset is relatively large, 
all  15 paired comparison graphs are complete, though possibly imbalanced. This makes it possible for us to obtain comparable results of these three sampling schemes. As in the second example, quality scores obtained from all the 43,266 paired comparisons are treated as the ground-truth. Figure \ref{online} shows mean $L^2$-distance of 100 times on LIVE \cite{LIVE} and IVC \cite{IVC} databases, and it is easy to find that all these reference images agree well with the theoretical and simulated results we have provided.

%
%

\subsection{Discussion}
In terms of the stability of HodgeRank, random sampling \emph{without} replacement exhibits a performance curve between the greedy sampling scheme, proposed by  \cite{osting2013enhanced,Osting:2012b}, and  random sampling \emph{with} replacement. When the sampling rate is sparse, greedy sampling dominates; when the sample size is increased, random sampling \emph{without} replacement is  indistinguishable from greedy sampling, both of which dominate the random sampling \emph{with} replacement (the simplest I.I.D. sampling).

Therefore, in practical situations, we suggest first to adopt greedy sampling in the initial stage which leads to a graph with large Fiedler value, then use random sampling \emph{without} replacement to approximate the results of greedy sampling. Such a transition point should depend on the graph vertex set size, $n$, for example $p_0 := \frac{2m}{(n-1)\log n}\approx \frac{n}{2\log n}$ which suggests $p_0\approx 3$ for $n=16$ in our simulated and real-world examples. After all, this random sampling scheme is simpler and more flexible than greedy sampling and does not significantly reduce the accuracy of the HodgeRank estimate.


\section {Conclusion} \label{sec:conclusion}

This paper analyzed two simple random sampling schemes for the HodgeRank estimate, including random sampling \emph{with} replacement and random sampling \emph{without} replacement. We showed that for a finite graph when it is sparse, random sampling
\emph{without} replacement approaches its performance
lower bound as random sampling \emph{with} replacement;
when it is dense, random sampling \emph{without} replacement
approaches its performance upper bound as
greedy sampling. For large graphs, such performance gaps are vanishing
in that all three sampling schemes exhibit similar
performance.

Because random sampling relies only on a random subset of
pairwise comparisons, data collection can be trivially
parallelized. This simple structure makes it easy to adapt to new situations, such
as online or distributed ranking. Based on these observations, we suggest in applications first adopt greedy sampling method in the initial stage and random sampling \emph{without} replacement in the second stage. For very large graphs, random sampling \emph{with} replacement may become the best choice, after all, it is the simplest I.I.D. sampling and when $n$ goes to infinity, the gaps among these sampling schemes vanish. The sampling schemes enable us to
derive reliable global ratings in an efficient manner, whence
provide us a helpful tool for those who exploit
crowdsourceable paired comparison data for subjective studies.

%

\section*{Acknowledgments}
The research of Qianqian Xu was supported in part by National Natural Science Foundation of China: 61402019, and China Postdoctoral Science Foundation: 2015T80025. The research of Jiechao Xiong and Yuan Yao was supported in part by National Basic Research Program of China under grant 2015CB85600, 2012CB825501, Baidu IBD, Microsoft Research Asia, and NSFC grant 61370004, 11421110001 (A3 project).

\vspace{1cm}

\section*{References}
{\footnotesize
\bibliographystyle{elsarticle-num} 
\bibliography{sigproc} }
\newpage
\appendix
\section{Proof of Theorem \ref{tm}} \label{proof}

The following basic inequality is used extensively throughout the proofs, which can be found in \cite{Hoeffding1963}.

\begin{lem}\label{lm3}
(Chernoff-Hoeffding theorem) Assume that $X_i\in[0,1],i=1,\dots,n$ are independent and $EX_i = \mu$. For $\epsilon>0$, the following inequalities hold
\[P(\bar{X}_n\le \mu-\epsilon)\le e^{-nKL(\mu-\epsilon||\mu)},\]
\[P(\bar{X}_n\ge \mu+\epsilon)\le e^{-nKL(\mu+\epsilon||\mu)},\]
where $KL(p||q) = p\log(p/q)+(1-p)\log((1-p)/(1-q))$, is Kullback-Leibler divergence, and $\bar{X}_n = \frac{1}{n}\sum_{i=1}^n X_i$ is the sample mean.
\end{lem}
\begin{coll}
\[P(\bar{X}_n\le k\mu)\le e^{-n\mu(k\log(k) - k + 1))}, k<1\]
\[P(\bar{X}_n\ge k\mu)\le e^{-n\mu(k\log(k) - k + 1))}, k>1\]
\end{coll}
\begin{proof}
$KL( k\mu||\mu) = k\mu\log(k) + (1-k\mu)\log(\frac{1-k\mu}{1-\mu})$. Defining $f(\mu) := (1-k\mu)\log(\frac{1-k\mu}{1-\mu}) + (k-1)\mu$, we compute
\[f(0) = 0, \quad f'(0) = 0, \ \text{ and } \  f''(\mu) = \frac{(1-k)^2}{(1-k\mu)(1-\mu)^2} > 0.\]
For all  $\mu\in(0,\min(1,1/k))$, we have that
\[f(\mu)> 0 \iff e^{-n  (1-k\mu)\log(\frac{1-k\mu}{1-\mu})} < e^{-n\mu(k-1) }.\] The result then follows from Lemma \eqref{lm3}.
\end{proof}

Throughout this section, $\E[\cdot]$ is used for expectation of random variables. Next, we  prove Lemma \ref{lm2}.

\begin{proof}[Proof of Lemma \ref{lm2}]
Our proof essentially follows \cite{FeigeOfek05} for the Erd\"{o}s-R\'{e}nyi random graph $G(n,p)$. For $G_0(n,m)$, consider $A = \sum_{k=1}^{m} A_k$ where $A_k$ is the adjacency matrix of I.I.D. edge samples. Hence
\[v^TAv = \sum_{k=1}^m v^TA_kv = \sum_{k=1}^m 2v_{i_k}v_{j_k}\]
 is the sum of I.I.D. variables.  Let $d := 2m/n$ denote  the expected degree of each graph vertex.

The proof strategy is as follows. To reach a bound for $\max_{v\in S}  v^T A v$, one needs a discrete cover $T$ of the set $S$. We turn to an upper bound $\max_{u,v\in T} u^T A v\geq c (1-\epsilon)^2 \sqrt{2m/n}$. However, any cover, $T$, has size $e^{O(n)}$ and therefore directly using Bernstein's inequality and the union bound doesn't work. Following \cite{FeigeOfek05}, we divide the set $\{(u_i,v_j)\colon u,v\in T\}$ into two parts: (1) light couples with $|u_i  v_j| \leq \sqrt{d}/n$, which can be bounded using Bernstein's inequality and (2) heavy couples with $| u_i   v_j | >\sqrt{d}/n$ but satisfying bounded degree and discrepancy properties. These two parts make up of the variation in $u^T A v$ which will lead to the bound in Lemma \ref{lm2}.

Following \cite{FeigeOfek05}, the first step is to reduce the set of vectors into a finite, yet exponentially large space. Let $S = \{v\bot1:\|v\|\le1\}$ and for fixed some $0 < \epsilon < 1$, define a grid which approximates $S$:
\[T=\left\{x\in\left(\frac{\epsilon}{\sqrt{n}}\ \Z\right)^n\colon \sum_i x_i =0,\|x\|\le1\right\}.\]

\noindent \textbf{Claim \cite{FeigeOfek05}.} The number of vectors in T is bounded by $e^{c_\epsilon n}$ for some $c_\epsilon$ which depends on $\epsilon$. If for every $u, v \in T~~ u^TAv \le c$, then for every $x \in S,~ x^TAx \le c/(1-\epsilon)^2$.

\medskip

\noindent It remains to show that

\medskip

\noindent \textbf{Claim.} $\exists c$, almost surely, $\forall u,v\in T, u^TAv \le c\sqrt{2m/n}.$

\medskip

\noindent  To prove this claim, we divide the set $\{(u_i,v_j)\colon u,v\in T\}$ into two parts:
 \begin{itemize}
 \item[(1)] light couples with $|u_i v_j| \leq \sqrt{d}/n$ and
 \item[(2)] heavy couples with $| u_i v_j| >\sqrt{d}/n$,
 \end{itemize}
 Let
 \begin{align*}
 L_k &= u_{i_k}v_{j_k}1_{\{|u_{i_k}v_{j_k}| \leq \frac{\sqrt{d}}{n}\}}+u_{j_k}v_{i_k}1_{\{|u_{j_k}v_{i_k}| \leq \frac{\sqrt{d}}{n}\}} \\
 H_k &= u_{i_k}v_{j_k}1_{\{|u_{i_k}v_{j_k}| > \frac{\sqrt{d}}{n}\}}+u_{j_k}v_{i_k}1_{\{|u_{j_k}v_{i_k}| > \frac{\sqrt{d}}{n}\}}.
 \end{align*}
 Then
\[u^TAv = \sum_{k=1}^m L_k + \sum_{k=1}^m H_k.\]

\bigskip

\noindent \textbf{Bound on the contribution of light couples.}
It's easy to compute
\begin{equation*}
\E L_k + \E H_k = \frac{2}{n(n-1)}\sum_{i\neq j}u_iv_j = -\frac{2}{n(n-1)} \langle u,v\rangle \nonumber
\end{equation*}
\begin{eqnarray*}
|\E H_k| = \frac{2}{n(n-1)} |\sum_{i\neq j}u_{i}v_{j}1_{\{|u_{i}v_{j}| \geq \frac{\sqrt{d}}{n}\}}|
  & \le & \frac{2}{n(n-1)}\sum_{i,j:|u_{i}v_{j}| \geq \frac{\sqrt{d}}{n}}| u_{i}^2v_{j}^2/\frac{\sqrt{d}}{n}| \\
  & \le & \frac{2}{(n-1)\sqrt{d}}\sum_{i}u_{i}^2 \sum_j v_{j}^2 \\
  & \le & \frac{2}{(n-1)\sqrt{d}}.
\end{eqnarray*}
So $|\E \sum_{k=1}^m L_k| \le m|\E H_k| + m |\E L_k + \E H_k| = O(\sqrt{d})$, as $d = 2m/n$.  We also have that
\begin{eqnarray*}
\mathrm{Var}(L_k) \le \E(L_k)^2 &\le& 2 \E \left((u_{i_k}v_{j_k}1_{\{|u_{i_k}v_{j_k}| \leq \frac{\sqrt{d}}{n}\}})^2 +¡¡(u_{j_k}v_{i_k}1_{\{|u_{j_k}v_{i_k}| \leq \frac{\sqrt{d}}{n}\}})^2 \right) \\\nonumber
&\le& \frac{4}{n(n-1)}\sum_{i}u^2_i \big( \sum_{j\neq i} v^2_j \big) \\\nonumber
&\le& \frac{4}{n(n-1)}.  \nonumber
\end{eqnarray*}
From definition, $|L_k|\le2\frac{\sqrt{d}}{n}$, so $|L_k-\E L_k| \le |L_k|+|\E L_k|\le 4\frac{\sqrt{d}}{n}\triangleq M$, Then Bernstein's inequality gives
\begin{eqnarray*}
 P\left(\sum_{k=1}^m L_k-\E\sum_{k=1}^m L_k>c\sqrt{d}\right) &\le& \exp\left(-\frac{c^2d}{2m \mathrm{Var}(L_k)+2Mc\sqrt{d}/3}\right)\\\nonumber
 &\le& \exp\left(-\frac{c^2d}{\frac{8m}{n(n-1)}+8cd/3n}\right)\\
 &\le& \exp(-O(cn)).
\end{eqnarray*}
So taking a union bound over $u,v \in T$, the contribution of light couples is bounded by $c\sqrt{d}$ with probability at least $1-e^{-O(n)}$.

\bigskip

\noindent \textbf{Bound on the contribution of heavy couples.}
As shown in \cite{FeigeOfek05}, if the random graph satisfies the bounded degree and discrepancy properties,
then the contribution of heavy couples is bounded by $O(\sqrt{d})$. We next define these two properties and prove that they hold.

\medskip

\noindent \textbf{Bounded degree property.} We say the \emph{bounded degree property} holds if every vertex has a degree bounded by $c_1 d$ (for some $c_1 > 1$). Using the fact $d_i \sim B(m,\mu),~\mu=2/n$, together with Lemma \ref{lm3} and $m>n\log n/2$, we have
\begin{equation*}
P(d_i\ge 6m/n) \le e^{-m\mu(3\log(3) - 2)}\le e^{-\frac{4m}{n}} \le n^{-2}.
\end{equation*}
So taking a union bound over $i$, we get with probability at least $1-1/n,~\forall i,~d_i\le 3d$.

\medskip

\noindent \textbf{Discrepancy property.}
Let $A,B\subseteq [n]$ be  disjoint and $e(A,B)$ be a random variable which denotes the number of edges between A and B. Then $e(A,B) \sim B(m,\frac{|A|\cdot|B|}{C_n^2})$. So, $\mu(A,B) = p|A|\cdot|B|$, with $p = \frac{m}{C_n^2}=\frac{d}{n-1}$, is the expected value of $e(A,B)$. Let $\lambda(A,B) = e(A,B)/\mu(A,B)$.  We say that the  \emph{dispcrepancy property} holds if there exists a constant $c$ such that for all $A, B \subseteq [n]$ with  $|B| \ge |A|$ one of the following holds:
\begin{enumerate}
  \item $\lambda(A,B) \le 4$,
  \item $e(A,B)\cdot \log \lambda(A,B) \le c\cdot |B| \cdot \log \frac{n}{|B|}$.
\end{enumerate}
We will show that the discrepancy property holds with  probability of at least $1-2/n$.
Write $a = |A|, b= |B|$, and suppose $b\ge a$. We assume the  bounded degree property holds with $c_1=3$ here.

\medskip

Case 1: $b > 3n/4$. Then $\mu(A,B) = ab \cdot \frac{d}{n-1} \ge \frac{3ad}{4}$. While $e(A,B)\le a\cdot 3d$ as each vertex in $A$ has degree bounded by $3d$, so $\lambda(A,B)\le 4$.

\medskip

Case 2: $b \le 3n/4$. Using the fact $e(A,B) \sim B(m,q)$, with $q=\frac{ab}{C_n^2}$ and Lemma \ref{lm3},  we have for $k\ge2$,
\begin{equation*}
P(e(A,B)\ge k\mu(A,B))\le e^{-mq(k\log(k) + (1-k))}\le e^{-\mu(A,B)k\log(k)/4}.
\end{equation*}
Then the union bound over all $A,B$ with size $a,b$ for $e(A,B)\ge k\mu(A,B)$ is
\begin{equation} \label{eq:case2_up1}
C_n^aC_n^be^{-\mu(A,B)k\log(k)/4}\le e^{-\mu(A,B)k\log(k)/4}\left(\frac{ne}{a}\right)^a\left(\frac{ne}{b}\right)^b.
\end{equation}
We want the right hand side of \eqref{eq:case2_up1} to be smaller than $1/n^3$, so it is enough to let
\begin{equation} \label{eq:case2_up2}
\mu(A,B)k\log(k)/4 \ge a\left(1+\log\frac{n}{a}\right)+b\left(1+\log\frac{n}{b}\right)+3\log n.
\end{equation}
Next, we are to give an upper bound for the right hand side of \eqref{eq:case2_up2}. Using the fact $x\log(n/x)$ is increasing in $(0,n/e)$ and decreasing in $(n/e,3n/4)$, we have for $b\le n/e$,
  \begin{equation*}
    a\left(1+\log\frac{n}{a}\right)+b\left(1+\log\frac{n}{b}\right)+3\log n    \le 4b\log\frac{n}{b}+3\log n \le 7b\log\frac{n}{b},
  \end{equation*}
and for $3n/4\ge b > n/e$,
  \begin{eqnarray*}
    a\left(1+\log\frac{n}{a}\right)+b\left(1+\log\frac{n}{b}\right)+3\log n    &\le& 2\cdot3n/4 + 2\cdot n/e+3\log n\\
    &\le& 11 \cdot3/4 \cdot \log(4/3)n \\
    &\le& 11b\log\frac{n}{b}.
  \end{eqnarray*}
Therefore to make \eqref{eq:case2_up2} valid, it suffices to assume $k\log k>\frac{44}{\mu(A,B)}b\log\frac{n}{b}$. Let $k_0\geq 2$ be the minimal number that satisfies this inequality. Using the union bound over all the possible $a, b$, we get the following conclusion. With probability of at least $1-1/n$, for every choice of A, B ($b \le 3n/4$) the following holds,
$$e(A,B)\le k_0 \mu(A,B).$$
If $k_0 = 2$ then we are done, otherwise $k_0\log(k_0)\mu(A,B) = O(1)b\log\frac{n}{b}$, so
\begin{equation*}
e(A,B)\cdot \log \lambda(A,B) \le e(A,B)\log(k_0)\le k_0 \log(k_0)\mu(A,B) = O(1)|B|\log\frac{n}{|B|},
\end{equation*}
as desired to satisfy the second condition for the discrepancy property.
\end{proof}

\bigskip

\begin{proof}[Proof of Theorem \ref{tm}]
The proof for $G(n,m)$ follows from \cite{Kolokolnikov2013},  which establishes the result for the Erd\"{o}s-R\'{e}nyi random model $G(n,p)$. Here we follow the same idea for $G_0(n,m)$, using an exponential concentration inequality (Lemma \ref{lm3}) to derive lower and upper bounds on $d_{\text{min}}$. The lower bound is directly from a union bound of independent argument. The upper bound deals  with weak dependence using the Chebyshev-Markov inequality.

Using Lemma \ref{lm1}, we only need to study the asymptotic limit of $\frac{d_{\text{min}}}{2m/n}$.
As $d_i \sim B(m,\mu), \mu=2/n$,  using Lemma \ref{lm3} we have
\begin{equation}\label{eq:prob-upper}
P(d_i\le 2am/n) \le e^{-m\mu(a\log(a) + (1-a))} = e^{\frac{2m}{n}\mathcal{H}(a)}.
\end{equation}
where $\mathcal{H}(a)=a-a\log(a)-1$.

In the other direction, suppose $i_0=2am/n$ is an integer, we have
\begin{eqnarray} \nonumber
P(d_i\le 2am/n)  &\ge& C_m^{i_0}(2/n)^{i_0}(1-2/n)^{m-i_0} \\\nonumber
    &\ge& \frac{(m-i_0)^{i_0}}{e\sqrt{i_0}(i_0/e)^{i_0}}(2/(n-2))^{i_0}(1-2/n)^{m} \\ \nonumber
    &=& \frac{1}{e\sqrt{i_0}}e^{i_0(1+\log((n/a-2)/(n-2)))+m\log(1-2/n)} \\
    \label{eq:logPoly}
    &\ge& \frac{1}{e\sqrt{i_0}}e^{i_0(1-\log(a))-2m/n-4m/n^2} \\ \nonumber
    &>& \frac{1}{e^3\sqrt{i_0}}e^{i_0(1-\log(a)-1/a)} \\ \nonumber
    &=& \frac{1}{e^3\sqrt{a\frac{2m}{n}}}e^{\frac{2m}{n}\mathcal{H}(a)}. \\ \nonumber
\end{eqnarray}
The inequality in \eqref{eq:logPoly} follows from $\log(1-x)\ge-x-x^2$ for all $x\in[0,2/3]$ and the assumption that $n\ge3$. The last inequality is due to $m<n^2/2$.

If $2am/n$ is not an integer, we can still have
\begin{equation}\label{eq:prob-lower}
P(d_i\le 2am/n) \ge \frac{c}{\sqrt{2m/n}}e^{\frac{2m}{n}\mathcal{H}(a)}.
\end{equation}
Equations (\ref{eq:prob-upper}) and (\ref{eq:prob-lower}) give an estimate for $P(d_i\le 2am/n)$.

Now let $a^{\pm} = a(p_0)\pm 1/\sqrt{2m/n}$, Taylor's theorem gives
\[\mathcal{H}(a^{\pm}) = \mathcal{H}(a(p_0)) \pm \frac{\mathcal{H}'(a_{\pm})}{\sqrt{2m/n}},\]
where $a_+\in(a(p_0),a(p_0)+1/\sqrt{2m/n})$, $a_-\in(a(p_0)-1/\sqrt{2m/n},a(p_0))$, and $\mathcal{H}'(a) = -\log(a)$ is the derivative of $\mathcal{H}$. Note $p_0\mathcal{H}(a(p_0))=-1$, so
\begin{equation*}
P\left(d_{\text{min}} \le a(p_0)\frac{2m}{n}-\sqrt{\frac{2m}{n}}\right) \le ne^{(2m/n)\mathcal{H}(a^-)}= e^{-\sqrt{2m/n}\mathcal{H}'(a_{-})} = O(e^{-\Omega(\sqrt{2m/n})}). \nonumber
\end{equation*}
Therefore, with probability at least $1-O(e^{-\Omega(\sqrt{2m/n})})$,
\begin{equation}
\label{eq:dminGEQ}
d_{\text{min}} \ge a(p_0)\frac{2m}{n}-\sqrt{\frac{2m}{n}}.
\end{equation}

Now we prove the reverse direction. Let $f_n = P(d_{\text{min}} \le a^+\frac{2m}{n})$, $X_i = \mathbf{1}_{\{d_i\le a^+\frac{2m}{n}\}}$ and $N_0=\sum_{i=1}^nX_i$. So $\mu_0 = \E N_0 = nf_n$. Chebyshev's inequality implies that
\begin{equation} \label{eq:cheby}
P(|N_0-\mu_0|>nf_n/2) \le \frac{4 \mathrm{Var}(N_0)}{n^2f_n^2},
\end{equation}
\begin{equation*}
\mathrm{Var}(N_0) = \sum_{i=1}^n \mathrm{Var}(X_i) +2\sum_{i<j} \mathrm{Cov}(X_i,X_j)= nf_n(1-f_n) + n(n-1) \mathrm{Cov}(X_1,X_2).
\end{equation*}
Next, we are going to claim
$$\mathrm{Cov}(X_1,X_2)\le O(1)pf_n^2,$$
\emph{i.e.} $P(X_1=1,X_2=1)\le (1+O(1)p)f_n^2$.

It is enough to prove $\forall k\le a^+\frac{2m}{n}$
\[P(d_1\le a^+\frac{2m}{n}|d_2=k) \le (1+O(1)p)f_n.\]
Note the conditional distribution of $d_1$ given $d_2=k$ is
$$B(k,1/(n-1)) + B(m-k,2/(n-1)),$$
\begin{align*}
&P(d_1\le a^+\frac{2m}{n}|d_2=k) \\
&=\sum_{i=0}^{k}C_k^i(\frac{1}{n-1})^i(\frac{n-2}{n-1})^{k-i} \sum_{j=0}^{a^+\frac{2m}{n}-i}C_{m-k}^{j}(\frac{2}{n-1})^j(\frac{n-3}{n-1})^{m-k-j}\\ \nonumber
&=\sum_{s=0}^{a^+\frac{2m}p{n}}\sum_{i=0}^{k\wedge s}C_k^i(\frac{1}{n-1})^i(\frac{n-2}{n-1})^{k-i}
C_{m-k}^{s-i}(\frac{2}{n-1})^{s-i}(\frac{n-3}{n-1})^{m-k-s+i}\\ \nonumber
&=\sum_{s=0}^{a^+\frac{2m}{n}}\sum_{i=0}^{k\wedge s}C_k^iC_{m-k}^{s-i}(\frac{1}{2})^i(\frac{n-2}{n-3})^{k-i}
(\frac{2}{n-1})^{s}(\frac{n-3}{n-1})^{m-s}\\ \nonumber
 &\le\sum_{s=0}^{a^+\frac{2m}{n}}\sum_{i=0}^{k\wedge s}C_k^iC_{m-k}^{s-i}(\frac{n-2}{n-3})^{k}
(\frac{n}{n-1})^s(\frac{2}{n})^{s}(\frac{n-2}{n})^{m-s}\\ \nonumber
  &\le\sum_{s=0}^{a^+\frac{2m}{n}}C_m^s(\frac{n}{n-1})^{s+k}(\frac{2}{n})^{s}(\frac{n-2}{n})^{m-s}\\ \nonumber
  &\le(\frac{n}{n-1})^{2a^+\frac{2m}{n}}\sum_{s=0}^{a^+\frac{2m}{n}}C_m^s(\frac{2}{n})^{s}(\frac{n-2}{n})^{m-s}\\ \nonumber
  &=(1+O(1)p)f_n.
\end{align*}

Hence, we get $\mathrm{Var}(N_0)\le nf_n+O(1)n^2f_n^2p$, and \eqref{eq:cheby} gives
\[P(|N_0-\mu_0|>nf_n/2) \le \frac{4}{nf_n}+4O(1)p.\]
Note that $nf_n \ge \frac{c}{\sqrt{2m/n}}e^{\sqrt{2m/n}\mathcal{H}'(a_{+})} \rightarrow \infty$, so with probability at least $1-O(e^{-\Omega(\sqrt{2m/n})})$, the graph has at least $nf_n/2\rightarrow\infty$ vertices satisfying
\[d_{i} \le a(p_0)\frac{2m}{n}+\sqrt{\frac{2m}{n}}.\]
Clearly $d_{\text{min}}$ also satisfies this statement. Combining this result with \eqref{eq:dminGEQ}, we have that with probability at least $1-O(e^{-\Omega(\sqrt{2m/n})})$,
\[|d_{\text{min}} - a(p_0)\frac{2m}{n}| \le \sqrt{\frac{2m}{n}},\]
as desired.
\end{proof}


%

%
\end{document}